\documentclass{article}

\usepackage[final]{neurips_2019} 

\usepackage[utf8]{inputenc} 
\usepackage[T1]{fontenc}    
\usepackage[colorlinks]{hyperref}       
\usepackage{url}            
\usepackage{booktabs}       
\usepackage{amsfonts}       
\usepackage{nicefrac}       
\usepackage{microtype}      
\usepackage{comment}
\usepackage{times}
\usepackage{epsfig}
\usepackage{graphicx}
\usepackage{amsmath}
\usepackage{amssymb}
\usepackage{amsthm}
\usepackage{multirow}
\usepackage{hyperref}
\usepackage{multicol}
\usepackage{algorithmic}
\usepackage{caption,subcaption}
\usepackage[vlined,ruled,linesnumbered]{algorithm2e}
\usepackage{appendix}
\usepackage{wrapfig}
\usepackage{xcolor}
\usepackage{footmisc}

\newcommand{\eg}{\emph{e.g.}}
\newcommand{\ie}{\emph{i.e.}}
\newcommand{\etal}{\emph{et~al.}}

\newtheorem{theorem}{Theorem}
 
\title{Efficient Communication in Multi-Agent Reinforcement Learning via Variance Based Control}

\author{
  Sai Qian Zhang \\
  Harvard University\\
   \And
   Qi Zhang \\
   Amazon Inc. \\
   \And
   Jieyu Lin \\
   University of Toronto \\
}

\begin{document}

\maketitle

\begin{abstract}
Multi-agent reinforcement learning (MARL) has recently received considerable attention due to its applicability to a wide range of real-world applications. However, achieving efficient communication among agents has always been an overarching problem in MARL. In this work, we propose Variance Based Control (VBC), a simple yet efficient technique to improve communication efficiency in MARL. By limiting the variance of the exchanged messages between agents during the training phase, the noisy component in the messages can be eliminated effectively, while the useful part can be preserved and utilized by the agents for better performance. 
Our evaluation using a challenging set of StarCraft II benchmarks indicates that our method achieves $2-10\times$ lower in communication overhead than state-of-the-art MARL algorithms, while allowing agents to better collaborate by developing sophisticated strategies.
\end{abstract}
\section{Introduction}

Many real-world applications (\eg, autonomous driving~[16], game playing~[12] and robotics control~[9]) today require reinforcement learning tasks to be carried out in multi-agent settings. 
In MARL, multiple agents interact with each other in a shared environment. Each agent only has access to partial observations of the environment, and needs to make local decisions based on partial observations as well as both direct and indirect interactions with the other agents.
This complex interaction model has introduced numerous challenges for MARL. In particular, during the training phase, each agent may dynamically change its strategy, causing dynamics in the surrounding environment and instability in the training process. Worse still, each agent can easily overfit its strategy to the behaviours of other agents~[11], which seriously deteriorates the overall performance. 


In the research literature, there have been three lines of research that try to mitigate the instability and inefficiency caused by decentralized execution. The most common approach is independent Q-learning (IQL)~[20], which breaks down a multi-agent learning problem into multiple independent single-agent learning problems, and makes each agent learn and act independently. Unfortunately, this approach does not account for instability caused by 
environment dynamics, and therefore often leads to collided actions during the execution. 
The second approach adopts the centralized training and decentralized execution~[18] paradigm, where a joint action value function is learned during the training phase to better coordinate the agents' behaviours. During execution, each agent acts independently without direct communication. The third approach introduces communication among agents during execution~[17,3]. This approach allows each agent to dynamically adjusts its strategy based on its local observation along with the information received from the other agents. Nonetheless, it introduces
additional communication overhead in terms of latency and bandwidth during execution, and its success is heavily dependent on the usefulness of the received information.

In this work, we leverage the advantages of both the second and third approach.
Specifically, we consider a fully cooperative scenario where all the agents collaborate to achieve a common objective. The agents are trained in a centralized fashion within the multi-agent Q-learning framework, and are allowed 
to communicate with each other during execution. However, 
unlike previous work, we have a few key insights. First, for many applications, it is often superfluous for an agent to wait for feedback from all surrounding agents before making an action decision. For instance, when the front camera on a autonomous vehicle detects an obstacle within the dangerous distance limit, it triggers the `brake` signal without waiting for the feedback from the other parts of the vehicle. Second, the feedback received from the other agents may not always provide useful information. For example, the navigation system of the autonomous vehicle should pay more attention to the messages sent by the perception system (\eg, camera, radar), and less attention to the entertainment system inside the vehicle before taking its action. Full communication among the agents leads to a large communication overhead and latency, which is impractical for the real system implementation with strict latency requirement and bandwidth limit (e.g., real-time traffic signal control, autonomous driving, etc). In addition, as pointed out by~[7], excessive amount of communication may introduce useless and even harmful information which could even impair the convergence of learning.

Motivated by these observations, we design a novel deep MARL architecture that dramatically improves inter-agent communication efficiency.
Specifically, 
we introduce Variance Based Control (VBC), a simple yet efficient approach to reduce the information transferred between the agents. By inserting an extra loss term on the variance of the exchanged information, the valuable and informative part inside the messages can be effectively extracted and leveraged to benefit the training of each individual agent. Unlike previous work, we do not introduce an extra decision module to dynamically adjust the communication pattern. This allows us to reduce the model complexity significantly. Instead, each agent first makes a preliminary decision based on its local information, and initiates communication only when its confidence level on this preliminary decision is low. Similarly, upon receiving the communication request, the agent replies to the request only when its message is informative. This architecture significantly improves the performance of the agents and reduces the communication overhead during the execution. Furthermore, it can be theoretically shown that the resulting training algorithm achieves guaranteed stability. 

For evaluation, we test VBC on the StarCraft Multi-Agent Chanllenge~[15], the results demonstrate that VBC achieves $20\%$ higher for winning rates and $2-10\times$ lower for communication overhead on average compared with the other benchmark algorithms. A video demo is available at~[2] for better illustration of the VBC performance. The code is available at \url{https://github.com/saizhang0218/VBC}.

\section{Related Work}
The simplest training method for MARL is to make each agent learn independently using Independent Q-Learning (IQL)~[20]. Although IQL is successful in solving simple tasks such as Pong~[19], it ignores the environment dynamics arose from the interactions among the agents. As a result, it suffers from the problem of poor convergence, making it difficult to handle advanced tasks.

Given the recent success on deep Q-learning~[12], some work explores the scheme of centralized training and decentralized execution. Sunehag~\etal~[18] propose Value Decomposition Network (VDN), a method that acquires the joint action value function by summing all the action value functions of each agent. All the agents are trained as a whole by updating the joint action value functions iteratively. QMIX~[14] sheds some light on VDN, and utilizes a neural network to represent the joint action value function as a function of the individual action value functions and the global state information. The authors of~[10] extends the actor-critic methods to the multi-agent scenario. By performing centralized training and decentralized execution over the agents, the agents can better adapt to the changes on the environment and collaborate with each other. Foerster~\etal~[5] propose counterfactual multi-agent policy gradient (COMA), which employs a centralized critic function to estimate the action value function of the joint, and decentralized actor functions to make each agent execute independently. All the aforementioned methods assume no communication between the agents during the execution. As a result, many subsequent approaches, including ours, can be applied to improve the performance of these methods. 

Learning the communication pattern for MARL is first proposed by Sukhbaatar~\textit{et.~al.}~[17]. The authors introduce \emph{CommNet}, a framework that adopts continuous communication for fully cooperative tasks. During the execution, each agent takes their internal states as well as the means of the internal states of the rest agents as the input to make decision on its action. 
The BiCNet~[13] uses a bidirectional coordinated network to connect the agents. However, both schemes require all-to-all communication among the agents, which can cause a significant communication overhead and latency.  

Several other proposals~[3,7,8] use a selection module to dynamically adjust the communication pattern between the agents. In~[3], the authors propose DIAL (Differentiable Inter-Agent Learning), the messages produced by an agent are selectively sent to the neighboring agents by the discretise/regularise unit (DRU). By jointly training DRU with the agent network, the communication overhead can be efficiently reduced. Jiang~\textit{et.~al.}~[7] propose an attentional communication model that learns when the communication is required and how to aggregate the shared information. However, an agent can only talk with the agents in its observable field at each timestep. This limits the speed of information propagation, and restricts the possible communication patterns when the local observable field is small. Kim~\textit{et.~al.} ~[8] propose a communication scheduling scheme for wireless environment, but only part of agents can broadcast their messages at each time.
In comparison, our approach does not apply the hard constraints on the communication pattern, which is advantageous to the learning process. Also our method does not adopt additional decision module for the communication scheduling, which greatly reduces the model complexity.

\section{Background}
\textbf{Deep Q-networks:}
We consider a standard reinforcement learning problem based on Markov Decision Process (MDP). At each timestamp $t$, the agent observes the state $s_{t}$, and chooses an action $a_{t}$. It then receives a reward $r_{t}$ for its action $a_{t}$ and proceeds to the next state $s_{t+1}$. The goal is to maximize the total expected discounted reward $R = \sum_{t=1}^{T} \gamma^{t} r_{t}$, where $\gamma\in [0,1]$ is the discount factor.
Deep Q-Networks (DQN) use a deep neural network to represent the action value function $Q_{\theta}(s,a) = E[R_{t}|s_{t} = s, a_{t} = a]$, where $\theta$ represents the parameters of the neural network, and $R_{t}$ is the total rewards received at and after $t$. During the training phase, a replay buffer is used to store the transition tuples $\big \langle s_{t},a_{t},s_{t+1},r_{t}\big \rangle$.
The action value function $Q_{\theta}(s,a)$ can be trained recursively by minimizing the loss $L = E_{s_{t},a_{t},r_{t},s_{t+1}}[y_{t}-Q_{\theta}(s_{t},a_{t})]^{2}$, where $y_{t} = r_{t} + \gamma max_{a_{t+1}} Q_{\theta'}(s_{t},a_{t+1})$ and $\theta'$ represents the parameters of the \emph{target network}. 
An action is usually selected with $\epsilon$-greedy policy. Namely, selecting the action with maximum action value with probability $1-\epsilon$, and choosing a random action with probability $\epsilon$.

\textbf{Multi-agent deep reinforcement learning:}
We consider an environment with $N$ agents work cooperatively to fulfill a given task.
At timestep $t$, each agent $i$ ($1\leq i \leq N$) receives a local observation $o_{i}^{t}$ and executes an action $a_{i}^{t}$. They then receive a joint reward $r_{t}$ and proceed to the next state. We use a vector $\textbf{a}_{t}=\{a_{i}^{t}\}$ to represent the joint actions taken by all the agents. The agents aim to maximize the joint reward by choosing the best joint actions $\textbf{a}_{t}$ at each timestep $t$. 

\textbf{Deep recurrent Q-networks:}
Traditional DQN generates action solely based on a limited number of local observations without considering the prior knowledge. 
Hausknecht~\etal~[6] introduce Deep Recurrent Q-Networks (DRQN), which models the action value function with a recurrent neural network (RNN). The DRQN leverages its recurrent structure to integrate the previous observations and knowledge for better decision-making.
At each timestep $t$, the DRQN $Q_{\theta}(o^{t}_{i},h^{t-1}_{i},a^{t}_{i})$ takes the local observation $o^{t}_{i}$ and hidden state $h^{t-1}_{i}$ from the previous steps as input to yield action values.

\textbf{Learning the joint Q-function:}
Recent research effort has been made on the learning of joint action value function for multi-agent Q-learning. Two representative works are VDN~[14] and QMIX~[18]. In VDN, the joint action value function $Q_{tot}(\textbf{o}_{t},\textbf{h}_{t-1},\textbf{a}_{t})$ is assumed to be the sum of all the individual action value functions $Q_{tot}(\textbf{o}_{t},\textbf{h}_{t-1},\textbf{a}_{t}) = \sum_{i}Q_{i}(o_{i}^{t}, h_{i}^{t-1},a_{i}^{t})$, where $\textbf{o}_{t} = \{o_{i}^{t}\}$, $\textbf{h}_{t} = \{h_{i}^{t}\}$ and $\textbf{a}_{t} = \{a_{i}^{t}\}$ are the collection of the observations, hidden states and actions of all the agents at timestep $t$ respectively. QMIX employs a neural network to represent the joint value function $Q_{tot}(\textbf{o}_{t},\textbf{h}_{t-1},\textbf{a}_{t})$ as a nonlinear function of $Q_{i}(o_{i}^{t}, h_{i}^{t-1},a_{i}^{t})$ and global state $s_{t}$. 

\section{Variance Based Control}
In this section, we present the detailed design of VBC in the context of multi-agent Q-learning. The main idea of VBC is to achieve the superior performance and communication efficiency by limiting the variance of the transferred messages. Moreover, 
during execution, each agent communicates with other agents only when its local decision is ambiguous. The degree of ambiguity is measured by the difference on top two largest action values. Upon receiving the communication request from other agents, the agent replies only if its feedback is informative, namely the variance of the feedback is high. 
\subsection{Agent Network Design}
\begin{figure}
    \centering
    \includegraphics[width=\columnwidth]{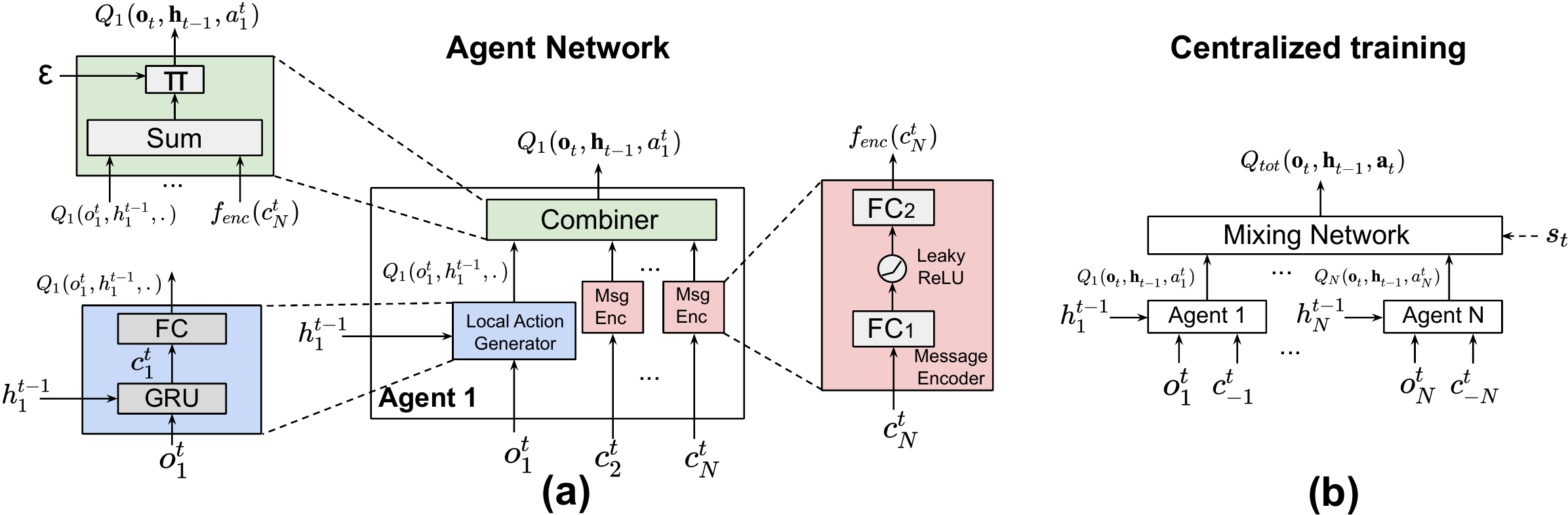}
    \caption{(a) Agent network structure of agent 1, which consists of local agent generator, combiner and several message encoder. (b) The mixing network takes the output $Q_{i}(\textbf{o}_{t},\textbf{h}_{t-1},a_{i}^{t})$ from each network of agent $i$, and perform centralized training. $c^{t}_{-i}$ means all the $c^{t}_{j\neq i}$.}
    \label{fig:network_system}
\end{figure}
The agent network consists of the following three networks: local action generator, message encoder and combiner. Figure~\ref{fig:network_system}(a) describes the network architecture for agent $1$. The \emph{local action generator} consists of a Gated Recurrent Unit (GRU) and a fully connected layer (FC). For network of agent $i$, the GRU takes the local observation $o_{i}^{t}$ and the hidden state $h_{i}^{t-1}$ as the inputs, and generates the intermediate results $c^{t}_{i}$. $c^{t}_{i}$ is then sent to the FC layer, which outputs the local action values $Q_{i}(o_{i}^{t},h_{i}^{t-1},a^{t}_{i})$ for each action $a^{t}_{i}\in A$, where $A$ is the set of possible actions.
The \emph{message encoder}, $f^{ij}_{enc}(.)$, is a multi-layer perceptrons (MLP) which contains two FC layers and a leaky ReLU layer. The agent network involves multiple independent message encoders, each accepts $c^{t}_{j}, j\neq i$ from another agent, and outputs $f_{enc}^{ij}(c^{t}_{j})$. The outputs from local action generator and message encoder are then sent to the \emph{combiner}, which produces the global action value function $Q_{i}(\textbf{o}_{t},\textbf{h}_{t-1},a_{i}^{t})$  of agent $i$ by taking into account the global observation $\textbf{o}_{t}$ and global history $\textbf{h}_{t-1}$. To simplify the design and reduce model complexity, we do not introduce extra parameters for the combiner. Instead, we make the dimension of the $f^{ij}_{enc}(c^{t}_{j})$ the same as the local action values $Q_{i}(o_{i}^{t},h_{i}^{t-1},.)$, and hence the combiner can simply perform elementwise summation over its inputs, namely $Q_{i}(\textbf{o}_{t},\textbf{h}_{t-1},.) = Q_{i}({o}_{i}^{t},{h}_{i}^{t-1},.) + \sum_{j\neq i} f^{ij}_{enc}(c^{t}_{j})$. The combiner chooses the action with the $\epsilon$-greedy policy, $\pi(.)$. 
Let $\theta_{local}^{i}$ and $\theta_{enc}^{ij}$ denote the set of parameters of the local action generators and the message encoders, respectively. To prevent the lazy agent problem~[18] and decrease the model complexity, we make $\theta_{enc}^{ij}$ is the same for all $i$ and $j(j\neq i)$, and also make $\theta_{local}^{i}$ the same for all $i$. Accordingly, we can drop the corner scripts and use $\boldsymbol\theta = \{\theta_{local},\theta_{enc}\}$ and $f_{enc}(.)$ to denote the agent network parameters and the message encoder. 
\subsection{Loss Function Definition}
During the training phase, the message encoder and local action generator learn jointly to generate the best estimation on the action values. More specifically, we employ a mixing network (shown in Figure~\ref{fig:network_system}(b)) to aggregate the global action value functions $Q_{i}(\textbf{o}_{t},\textbf{h}_{t-1},a_{i}^{t})$ from each agents $i$, and yields the joint action value function, $Q_{tot}(\textbf{o}_{t},\textbf{h}_{t-1},\textbf{a}_{t})$. 
To limit the variance of the messages from the other agents, we introduce an extra loss term on the variance of the outputs of the message encoders $f_{enc}(c^{t}_{j})$. The loss function during the training phase is defined as: 
\begin{equation}
    \small
    L(\theta_{local},\theta_{enc}) = \sum_{b=1}^{B}\sum_{t=1}^{T}\big[(y^{b}_{tot}-Q_{tot}(\textbf{o}_{t}^{b},\textbf{h}_{t-1}^{b},\textbf{a}_{t}^{b};\boldsymbol\theta))^{2} + \lambda \sum_{i=1}^{N} Var(f_{enc}(c^{t,b}_{i}))\big]
    \label{eqn:loss}
\end{equation}
where $y^{b}_{tot} = r_{t}^{b} + \gamma max_{\textbf{a}_{t+1}}Q_{tot}(\textbf{o}_{t+1}^{b},\textbf{h}_{t}^{b},\textbf{a}_{t+1};\boldsymbol\theta^{-})$, $\boldsymbol\theta^{-}$ is the parameter of the target network which is copied from the $\boldsymbol\theta$ periodically, $Var(.)$ is the variance function and $\lambda$ is the weight of the loss on it. $b$ is the batch index. 
The replay buffer is refreshed periodically by running each agent network and selecting the action which maximizes  $Q_{i}(o_{i}^{t},h_{i}^{t-1},.)$. 

\subsection{Communication Protocol Design}
During the execution phase, for every timestep $t$, the agent $i$ first computes the local action value function $Q_{i}(o_{i}^{t},h_{i}^{t-1},.)$ and $f_{enc}(c_{i}^{t})$. It then measures the confidence level on the local decision by computing the difference between the largest and the second largest element within the action values. An example is given in Figure~\ref{fig:comm_protocol}(a). Assume agent 1 has three actions to select, and the output of the local action generator of agent $1$ is $Q_{1}(o_{1}^{t},h_{1}^{t-1},.) = (0.1,1.6,3.8)$, and the difference between the largest and the second largest action values is $3.8-1.6 = 2.2$, which is greater than the threshold $\delta_{1} = 1.0$. Given the fact that the variance of message encoder outputs $f_{enc}(c^{t}_{j})$ from the agent 2 and 3 is relatively small, due to the additional penalty term on variance in equation~\ref{eqn:loss}, it is highly possible that the global action value function $Q_{1}(\textbf{o}_{1}^{t},\textbf{h}_{1}^{t-1},.)$ also has the largest value in its third element. Therefore agent 1 does not have to talk to other agents to acquire $f_{enc}(c^{t}_{j})$. On the other hand, agent 1 broadcasts a request to ask for help if its confidence level on the local decision is low. Because the request does not contain any actual data, it consumes very low bandwidth. Upon receiving the request, only the agents whose message has a large variance reply (Figure~\ref{fig:comm_protocol}(b)), since their messages may change the current action decision at agent 1. This protocol not only reduces the communication overhead considerably, but also eliminates less informative messages which may impair the performance. Each agent only consults with the other agents when its confidence level on the local decision is low, and the other agents only reply when their messages can potentially change the final decision. The detailed protocol and operations performed at the agent unit is summarized in Algorithm~\ref{alg:communication_proto}. 
\begin{figure}
    \centering
    \includegraphics[width=\columnwidth]{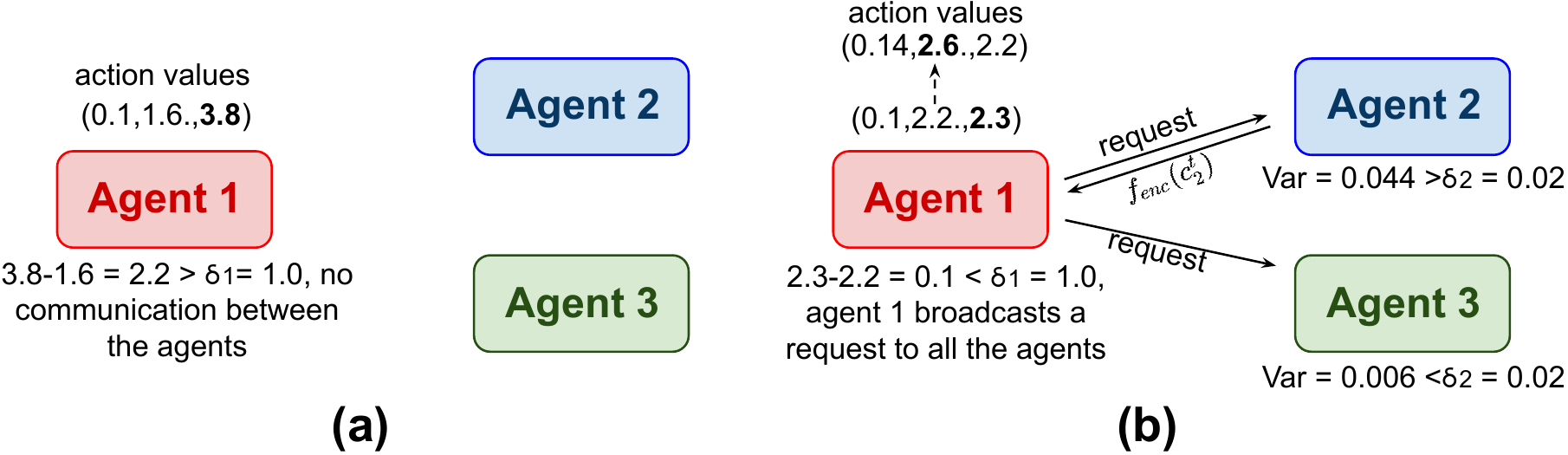}
    \caption{An example on communication protocol of the agents during execution.}
    \label{fig:comm_protocol}
\end{figure}
\begin{algorithm}[tp!]
{
  \small
  \textbf{Input}: Threshold on confidence of local actions
  $\delta_{1}$, threshold on variance of message encoder output $\delta_{2}$. Total number of agents N.\\
    \For {$t\in T$}{
         $\backslash\backslash$ \textbf{Decision on the action of itself}: \\
         Compute local action values $Q_{i}(o_{i}^{t},h_{i}^{t-1},.)$. Denote $m_{1},m_{2}$ the top two largest values of $Q_{i}(o_{i}^{t},h_{i}^{t-1},.)$. \\
        \If {$m_{1}-m_{2}\geq \delta_{1}$}{
            Let $Q_{i}(\textbf{o}_{t},\textbf{h}_{t-1},.)= Q_{i}(o_{i}^{t},h_{i}^{t-1},.)$.
        }
        \Else{
            Broadcast a request to the other agents, and receive the $f_{enc}(c_{j}^{t})$ from $N_{reply} (N_{reply}\leq N)$ agents. \\
            Let $Q_{i}(\textbf{o}_{t},\textbf{h}_{t-1},.)= Q_{i}(\textbf{o}_{i}^{t},\textbf{h}_{i}^{t-1},.) + \sum_{j=1}^{N_{reply}} f_{enc}(c_{j}^{t})$. 
        }
        $\backslash\backslash$ \textbf{Generation of the exchanged message for the other agents}: \\
        Calculate variance of $f_{enc}(c_{i}^{t})$, if $Var(f_{enc}(c_{i}^{t}))\geq \delta_{2}$, store $f_{enc}(c_{i}^{t})$ in the buffer. \\
        \If {$Var(f_{enc}(c_{i}^{t}))\geq \delta_{2}$ \textbf{and} Receive a request from agent $j$}{
            {
               Reply the request from agent j with $f_{enc}(c_{i}^{t})$. \\
            }
        }        
        
    } 
	\caption{Communication protocol at agent $i$}
	\label{alg:communication_proto}
}
\end{algorithm}
\section{Convergence Analysis}
In this section, we analyze convergence of the learning process with the loss function defined in equation~(\ref{eqn:loss}) under the tabular setting. For the sake of simplicity, we ignore the dependency of the action value function on the previous knowledge $\textbf{h}_{t}$. 
To minimize equation~(\ref{eqn:loss}), given the initial state $Q_{0}$, at iteration $k$, the $q$ values in the table is updated according to the following rule:
\begin{equation}
    \small
    Q_{tot}^{k+1}(\textbf{o}_{t},\textbf{a}_{t}) = Q_{tot}^{k}(\textbf{o}_{t},\textbf{a}_{t}) + \eta_{k}\bigg[r_{t} + \gamma max_{\textbf{a}}Q_{tot}^{k}(\textbf{o}_{t+1},\textbf{a})-Q_{tot}^{k}(\textbf{o}_{t},\textbf{a}_{t}) - \lambda\sum_{i=1}^{N}\frac{\partial Var(f_{enc}(c^{t}_{i}))}{\partial Q_{tot}^{k}(\textbf{o}_{t},\textbf{a}_{t})}\bigg] 
    \label{eqn:update_rule}
\end{equation}
where $\eta_{k}$, $Q_{tot}^{k}(.)$ are the learning rate and the joint action value function at iteration $k$ respectively, and $Q_{tot}^{*}(.)$ is the optimal joint action value function. We have the following result on the convergence of the learning process. A detailed proof is given in the appendix.
\begin{theorem}
    Assume $0\leq \eta_{k} \leq 1$, $\sum_{k}\eta_{k}=\infty$, $\sum_{k}\eta_{k}^{2}<\infty$. Also assume the number of possible actions and states are finite. By performing equation~\ref{eqn:update_rule} iteratively, we have $||Q_{tot}^{k}(\textbf{o}_{t},\textbf{a}_{t})-Q_{tot}^{*}(\textbf{o}_{t},\textbf{a}_{t})||\leq \lambda NG$ $\forall \textbf{o}_{t},\textbf{a}_{t}$, as $k \rightarrow \infty$, where $G$ satisfies $||\frac{\partial Var(f_{enc}(c^{t}_{i}))}{\partial Q_{tot}^{k}(\textbf{o}_{t},\textbf{a}_{t})}||\leq G, \forall i,k,t,\textbf{o}_{t},\textbf{a}_{t}$.
\end{theorem}

\section{Experiment}
We evaluated the performance of VBC with the StarCraft Multi-Agent Challenge (SMAC)~[15].
StarCraft II~[1] is a real-time strategy (RTS) game that has recently been utilized as a benchmark by the reinforcement learning community~[14,5,3,4]. In this work, we focus on the~\emph{decentralized micromanagement problem} in StarCraft II, which involves two armies, one controlled by the user (i.e. a group of agents), and the other controlled by the build-in StarCraft II AI.
The goal of the user is to control its allied units to destroy all enemy units, while minimizing received damage on each unit. 
We consider six different battle settings. Three of them are \emph{symmetrical battles}, where both the user and the enemy groups consist of 2 Stalkers and 3 Zealots (2s3z), 3 Stalkers and 5 Zealots (2s5z), and 1 Medivac, 2 Marauders and 7 Marines (MMM) respectively. The other three are \emph{unsymmetrical battles}, where the user and enemy groups have different army unit compositions, including: 3 Stalkers for user versus 4 Zealots for enemy (3s$\_$vs$\_$4z), 6 Hydralisks for user versus 8 Zealots for enemy (6s$\_$vs$\_$8z), and 6 Zealot for user versus 24 Zerglings for enemy (6z$\_$vs$\_$24zerg). The unsymmetrical battles is considered to be harder than the symmetrical battles because of the difference in army size.

At each timestep, each agent controls a single unit to perform an action, including move[direction], attack[enemy$\_$id], stop and no-op. Each agent has a limited \emph{sight range} and \emph{shooting range}, where shooting range is less than the sight range. The attack operation is available only when the enemies are within the shooting range. The joint reward received by the allied units equals to the total damage inflicted on enemy units. Additionally, the agents are rewarded 100 extra points after killing each enemy unit, and 200 extra points for killing the entire army. The user group wins the battle only when the allied units kill all the enemies within the time limit. Otherwise the built-in AI wins. 
The input observation of each agent is a vector consists the following information of each allied unit and enemy unit in its sight range: relative $x$, $y$ coordinates, relative distance and agent type. For the detailed game settings, hyperparameters, and additional experiment evaluation over other test environment, please refer to appendix.

\subsection{Results}
We compare VBC and other banchmark algorithms, including
VDN~[18], QMIX~[14] and SchedNet~[8], for controlling allied units. We consider two types of VBCs by adopting the mixing networks of VDN and QMIX, denoted as VBC+VDN and VBC+QMIX. The mixing network of VDN simply computes the elementwise summation across all the inputs, and the mixing network of QMIX deploys a neural network whose weight is derived from the global state $\textbf{s}_{t}$. The detailed architecture of this mixing network can be found in~[14]. Additionally, we remove the penalty in Equation~(\ref{eqn:loss}), and drop the limit on variance during the execution phase (\ie, $\delta_{1} = \infty$ and $\delta_{2} = -\infty$). The agents are trained with the same network architecture shown in Figure~(\ref{fig:network_system}), and the mixing network of VDN is used. We call this algorithm FC (full communication). For SchedNet, at every timestep only $K$ out of $N$ agents can broadcast their messages by using $\emph{Top(k)}$ scheduling policy~[8]. $K$ changes with different battles and we usually set $K$ close to $0.5N$, that is, each time roughly half of the allied units can broadcast their messages. The VBC are trained for different number of episodes based on the difficulties of the battles, which we describe in detail next.

To measure the speed of convergence for each algorithm, we stop the training process and save the current model every 200 training episodes. We then run 20 test episodes and measure the winning rates for these 20 episodes. For VBC+VDN and VBC+QMIX, the winning rates are measured by running the communication protocol described in Algorithm~\ref{alg:communication_proto}. For easy tasks, namely MMM and 2s$\_$vs$\_$3z, we train the algorithms with $2$ million and $4$ million episodes respectively. For all the other tasks, we train the algorithms with $10$ million episodes. Each algorithm is trained 15 times. Figure~\ref{fig:winning_rates} shows the average winning rate and 95$\%$ confidence interval of each algorithm for all the six tasks. 
For hyperparameters used by VBC (\ie, $\lambda$ used in equation~(\ref{eqn:loss}), $\delta_{1},\delta_{2}$ in Algorithm~\ref{alg:communication_proto}), we first search for a coarse parameter range based on random trial, experience and message statistics. We then perform a random search within a smaller hyperparameter space. Best selections are shown in the legend of each figure.

We notice that the algorithms that involve communication (\ie, SchedNet, FC, VBC) outperform the algorithms without communication (\ie, VDN, QMIX) in all the six tasks. This is a clear indication that communication benefits the performance. Moreover, both VBC+VDN and VBC+QMIX achieve better winning rates than SchedNet, because SchedNet only allows a fixed number of agents to talk at every timestep, which prohibits some key information to exchange timely. Finally, VBC achieves a similar performance as FC and even outplays FC for some tasks (\eg, 2s3z,6h$\_$vs$\_$8z, 6z$\_$vs$\_$24zerg). This is because a fair amount of communication between the agents are noisy and redundant. By eliminating these undesired messages, VBC is able to achieve both communication efficiency and performance gain. 

\begin{figure*}[tp!]
    \centering
    \begin{subfigure}{.333\textwidth}
        \centering
        \includegraphics[width=0.92\linewidth, height=0.7\linewidth]{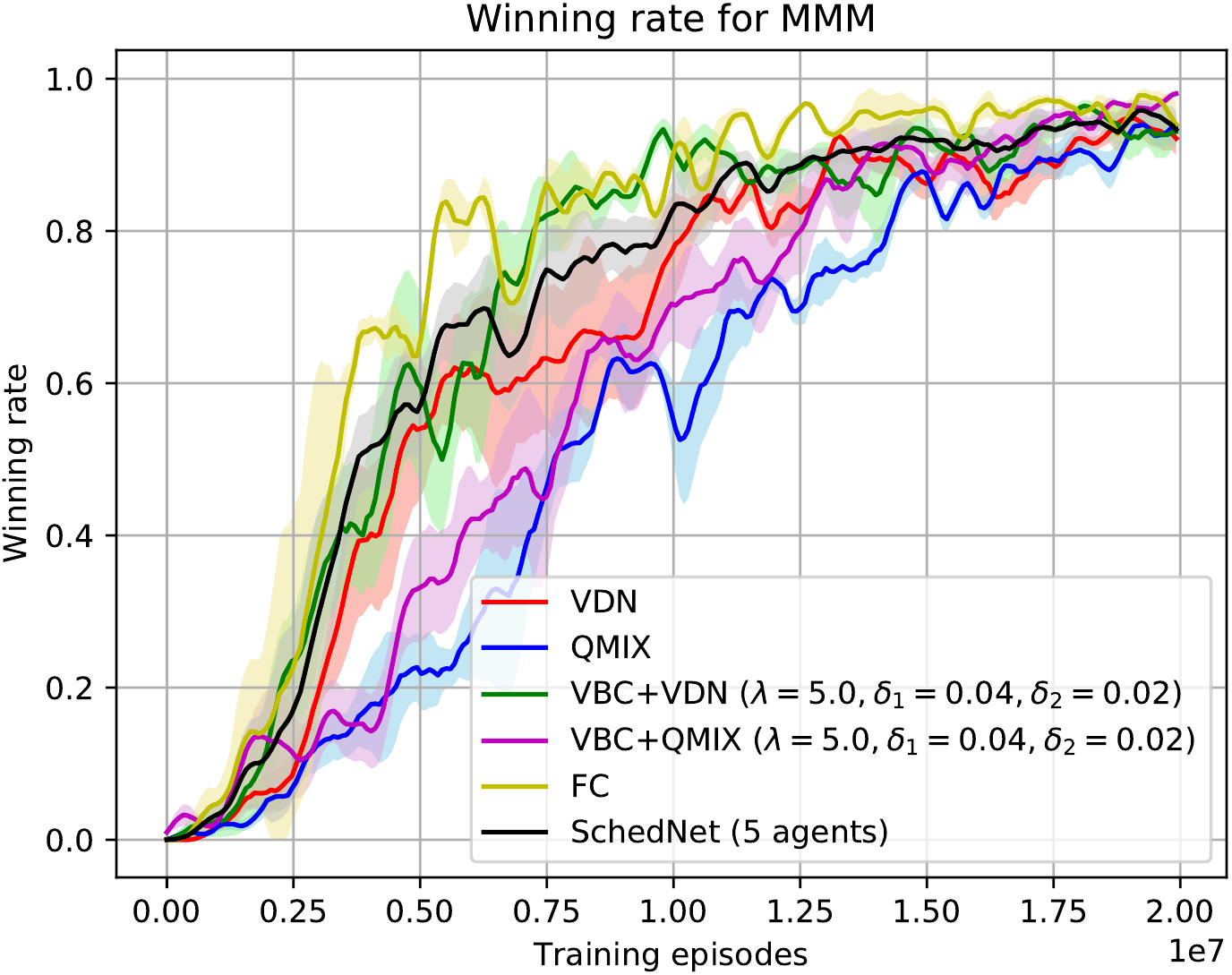}
        \subcaption{MMM}\label{fig:MMM}
    \end{subfigure}%
    \begin{subfigure}{0.33\textwidth}
        \centering
        \includegraphics[width=0.92\linewidth, height=0.7\linewidth]{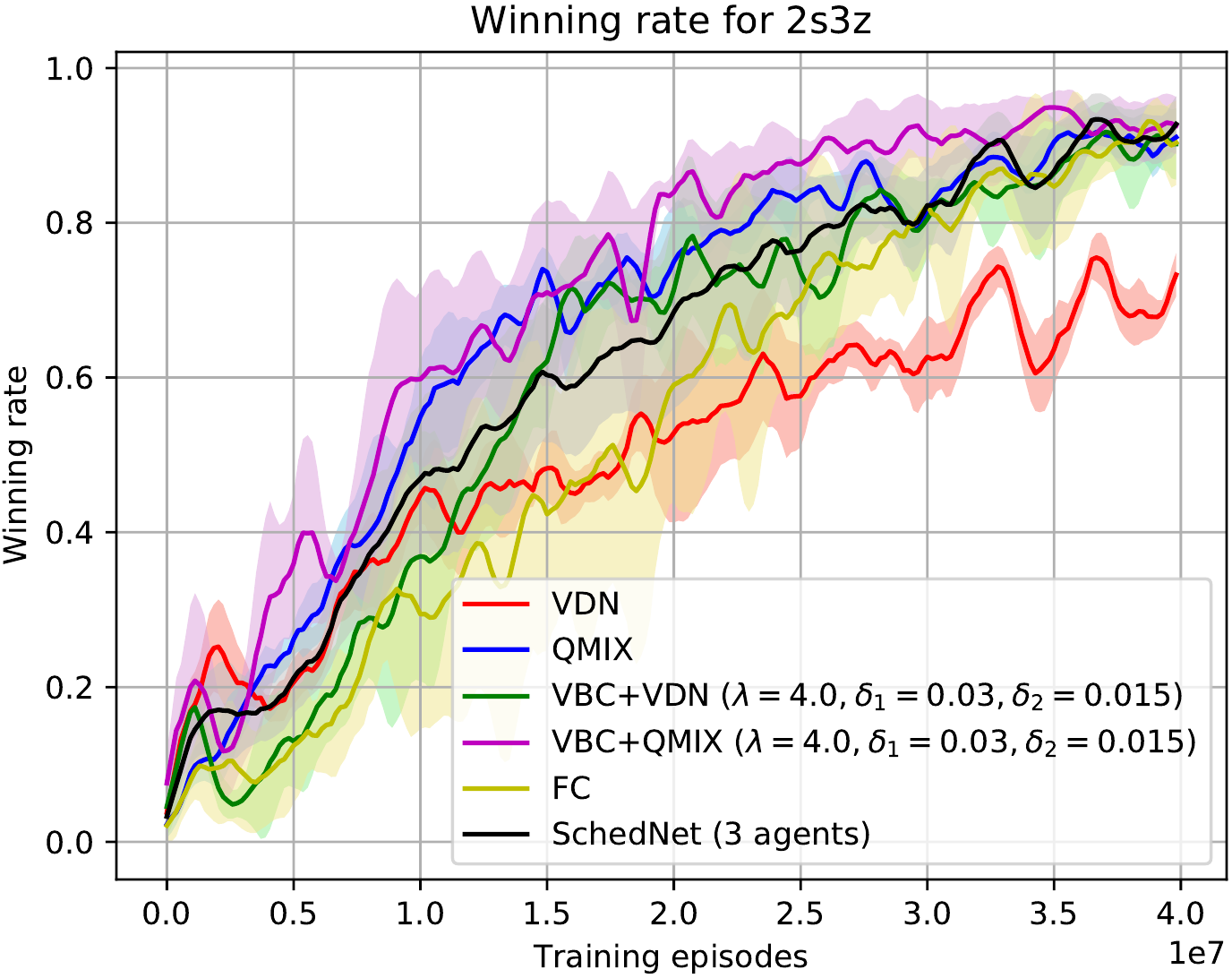}
        \subcaption{2s3z}\label{fig:2s3z}
    \end{subfigure}
    \begin{subfigure}{0.33\textwidth}
        \centering
        \includegraphics[width=0.92\linewidth, height=0.7\linewidth]{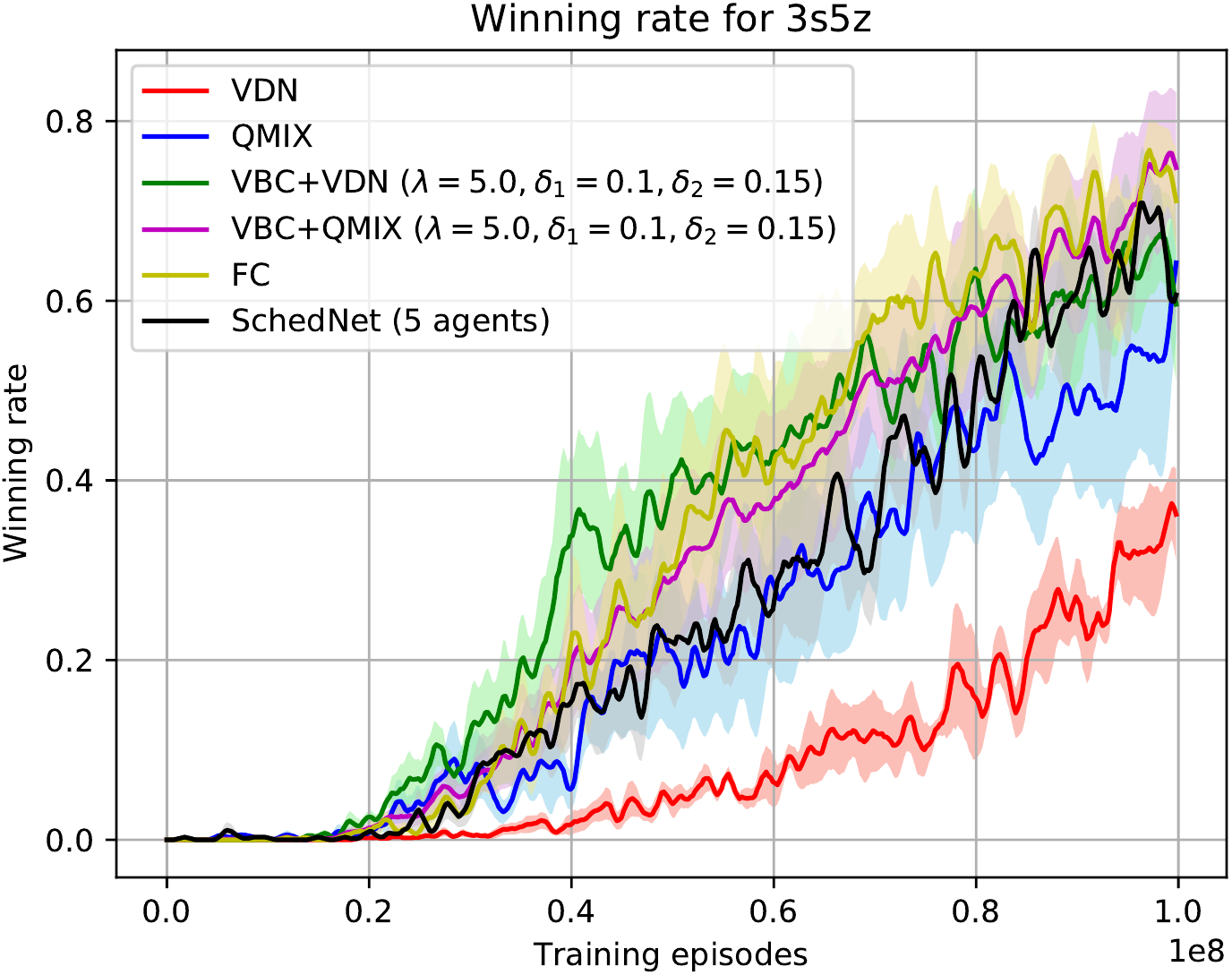}
        \subcaption{3s5z}\label{fig:3s5z}
    \end{subfigure}
    \begin{subfigure}{.333\textwidth}
        \centering
        \includegraphics[width=0.92\linewidth, height=0.7\linewidth]{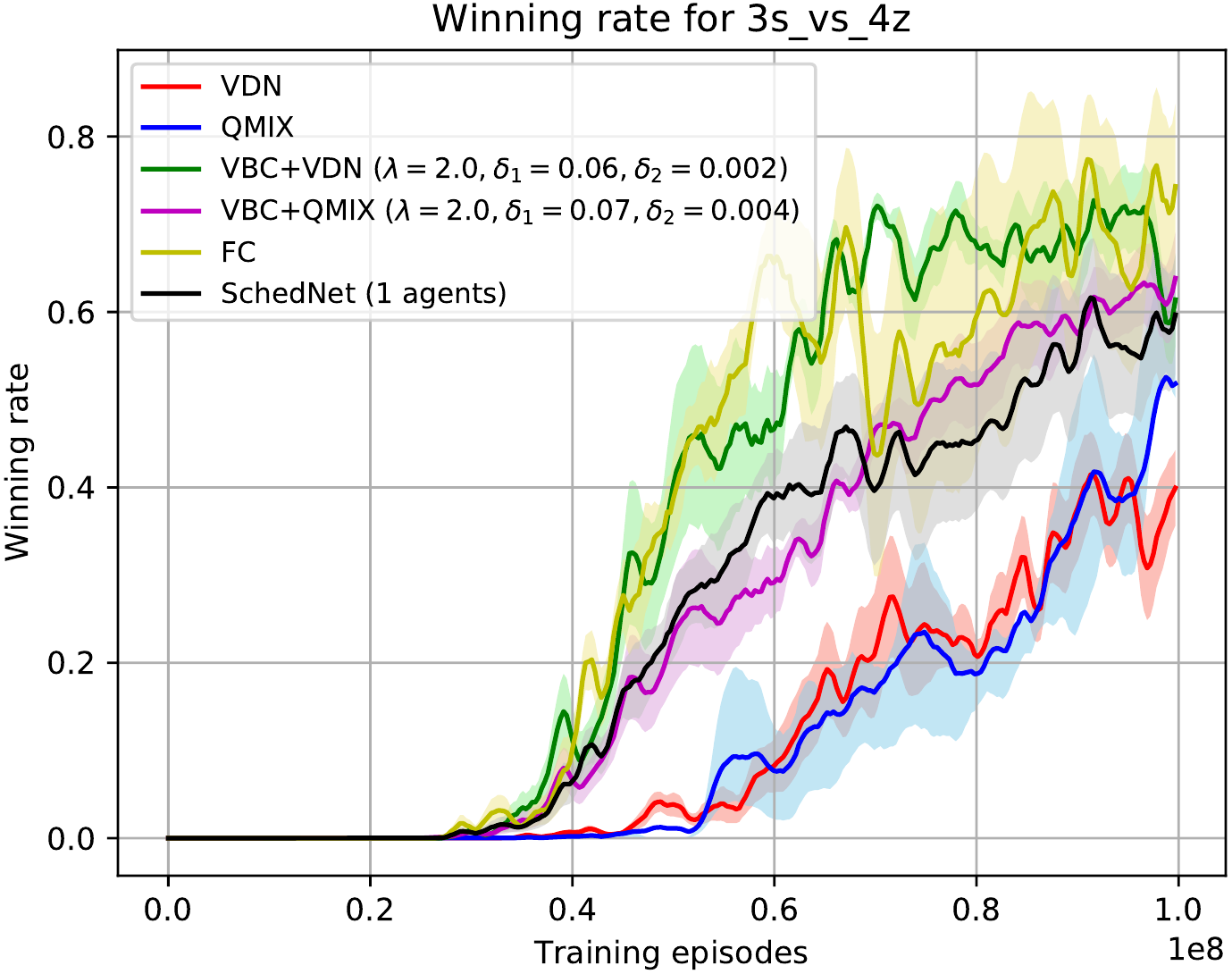}
        \subcaption{3s\_vs\_4z}\label{fig:3s_vs_4z}
    \end{subfigure}%
    \begin{subfigure}{0.33\textwidth}
        \centering
        \includegraphics[width=0.92\linewidth, height=0.7\linewidth]{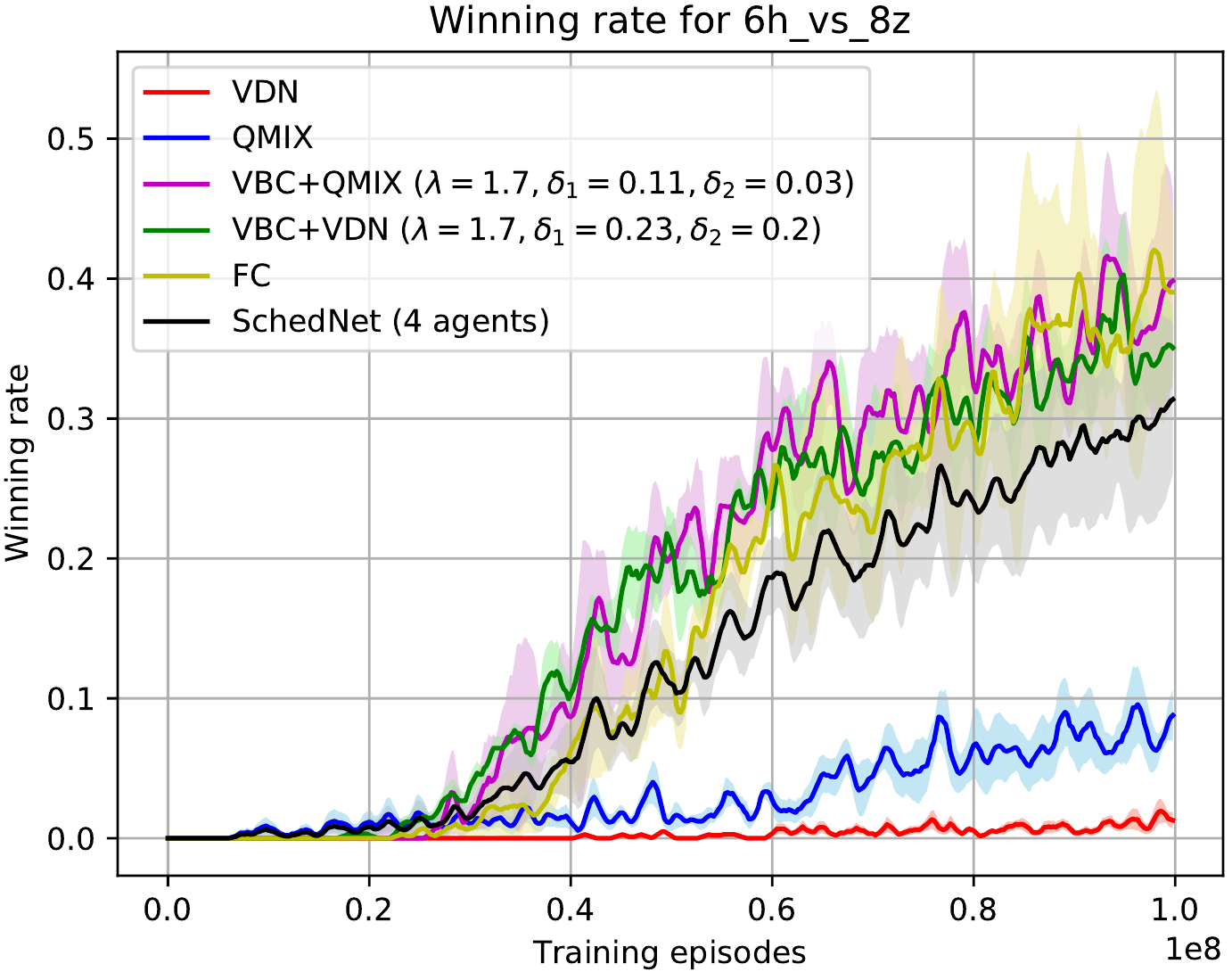}
        \subcaption{6h\_vs\_8z}\label{fig:6h_vs_8z}
    \end{subfigure}
    \begin{subfigure}{0.33\textwidth}
        \centering
        \includegraphics[width=0.92\linewidth, height=0.7\linewidth]{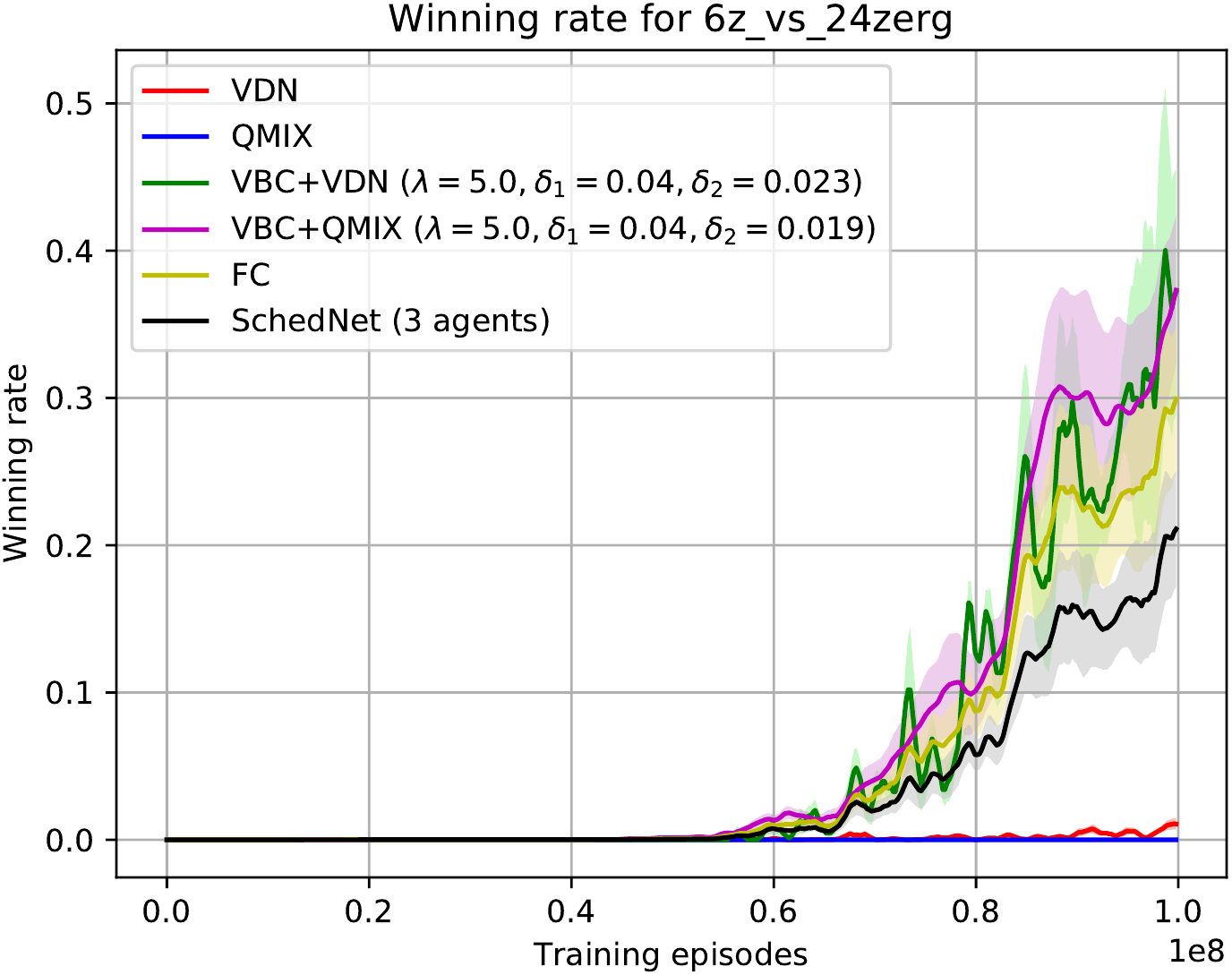}
        \subcaption{6z\_vs\_24zerg}\label{fig:6z_vs_24zerg}
    \end{subfigure}  
    \caption{Winning rates for the six tasks, the shaded regions represent the 95\% confidence intervals.}
    \label{fig:winning_rates}
\end{figure*}
\begin{figure*}[htb]
    \centering
    \begin{subfigure}{.3\textwidth}
        \centering
        \includegraphics[width=0.85\linewidth, height=0.6\linewidth]{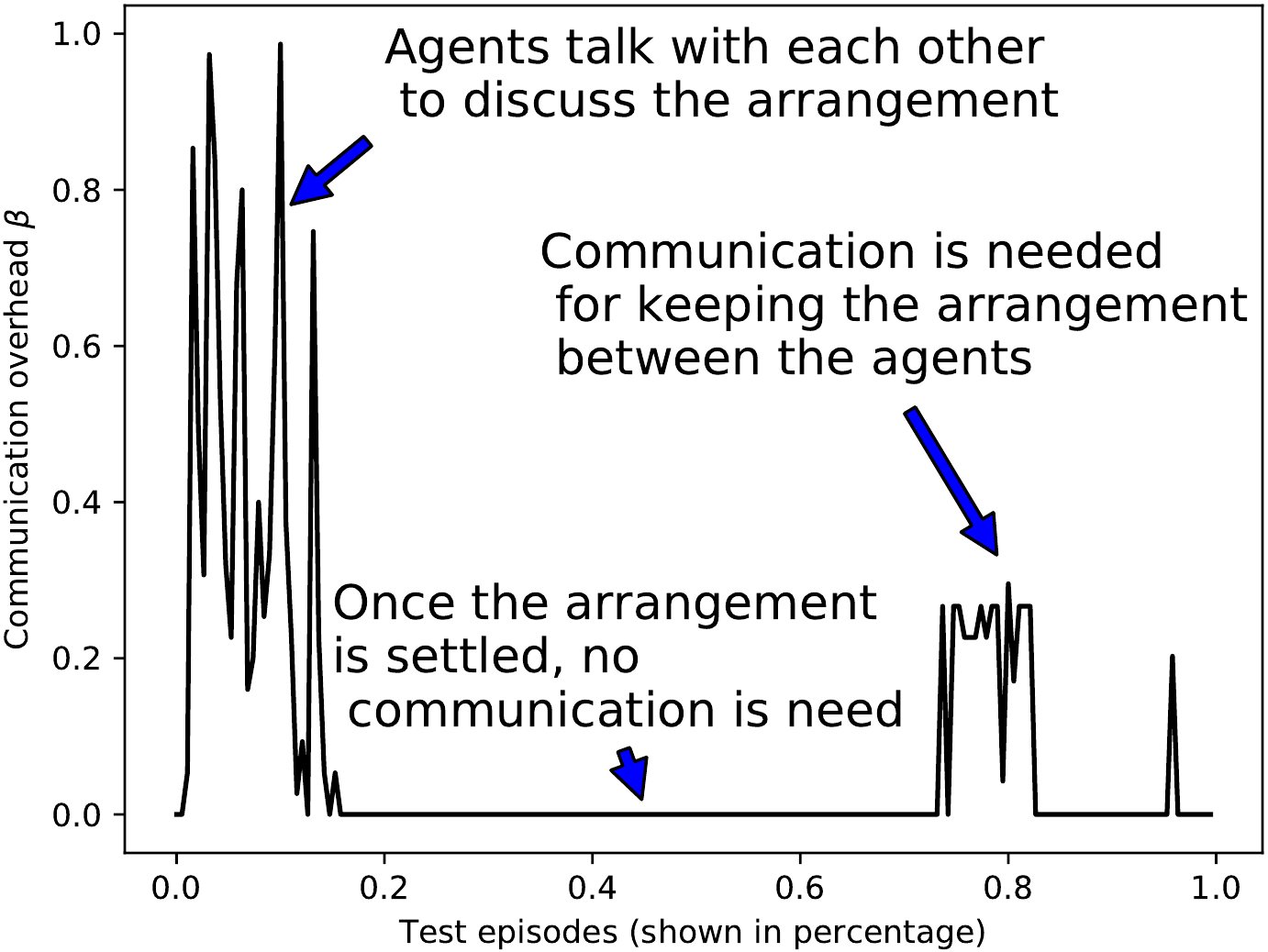}
        \caption{Communication overhead}
    \end{subfigure}%
    \begin{subfigure}{0.3\textwidth}
        \centering
        \includegraphics[width=0.78\linewidth, height=0.6\linewidth]{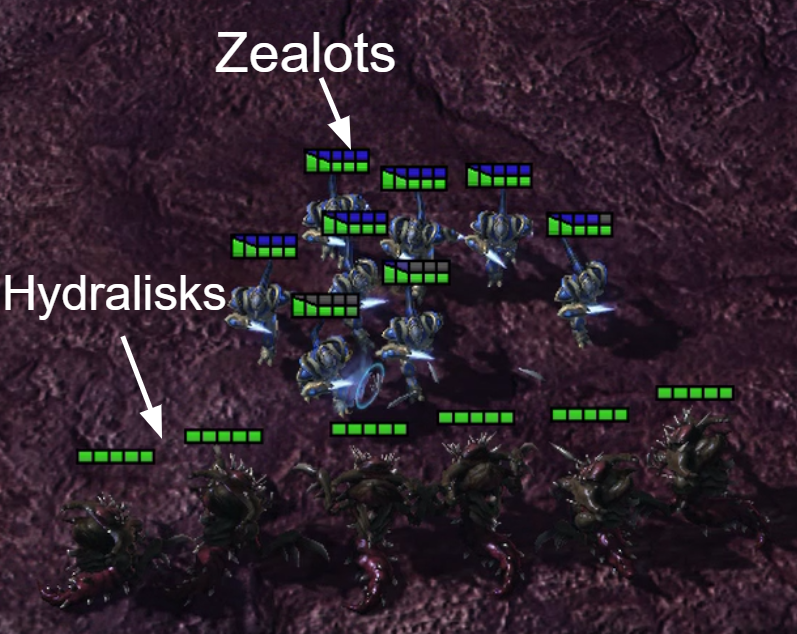}
        \caption{VBC (6h\_vs\_8z)}
    \end{subfigure}
    \begin{subfigure}{0.3\textwidth}
        \centering
        \includegraphics[width=0.85\linewidth, height=0.6\linewidth]{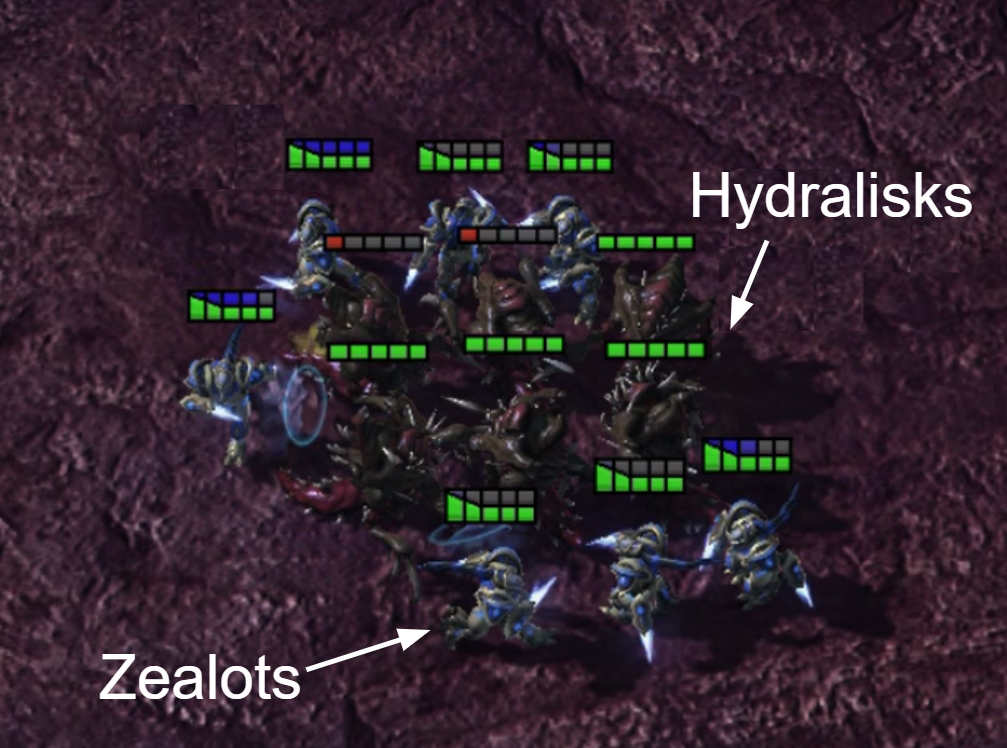}
        \caption{QMIX and VDN (6h\_vs\_8z)}
    \end{subfigure}
    \begin{subfigure}{0.3\textwidth}
        \centering
        \includegraphics[width=0.85\linewidth, height=0.65\linewidth]{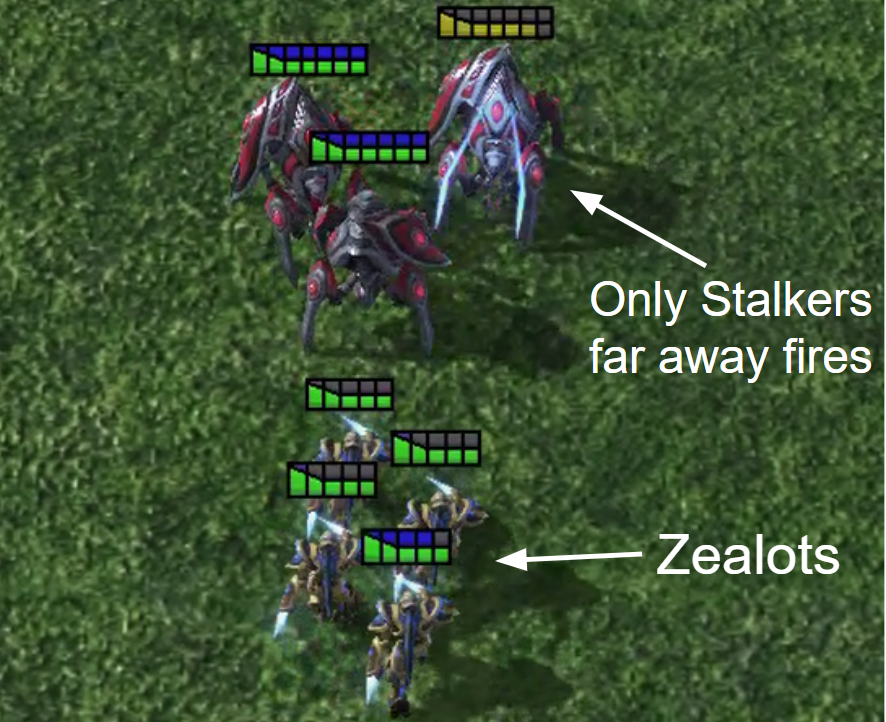}
        \caption{VBC (3s$\_$vs$\_$4z)}
    \end{subfigure}
    \begin{subfigure}{0.3\textwidth}
        \centering
        \includegraphics[width=0.8\linewidth, height=0.65\linewidth]{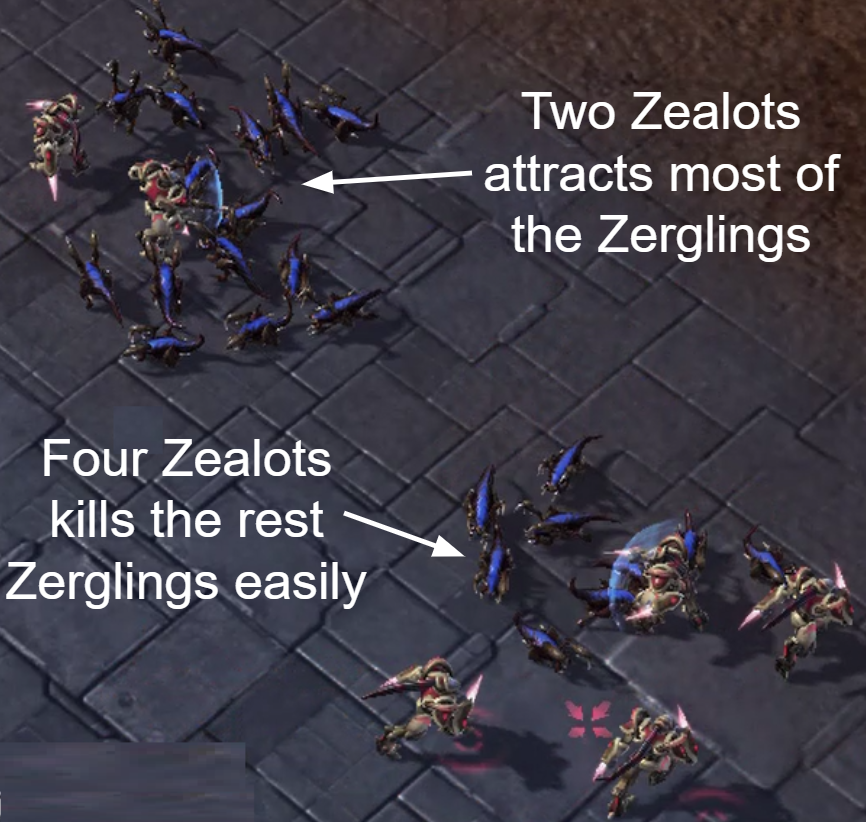}
        \caption{VBC (6z\_vs\_24zergs, stage 1)}
    \end{subfigure}
    \begin{subfigure}{0.3\textwidth}
        \centering
        \includegraphics[width=0.83\linewidth, height=0.65\linewidth]{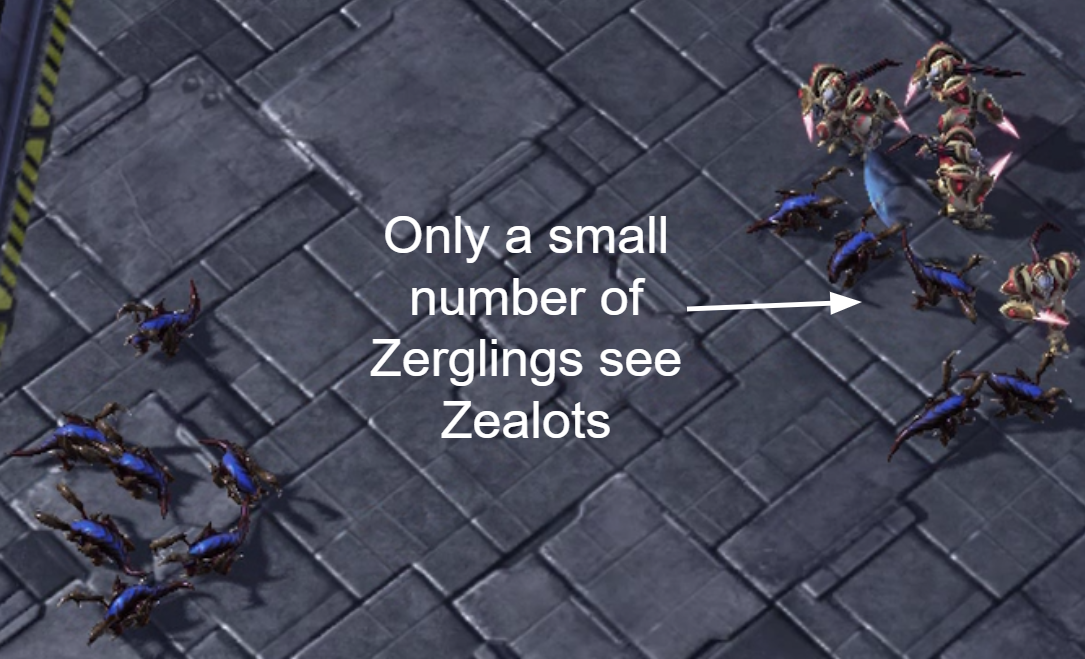}
        \caption{VBC (6z\_vs\_24zergs, stage 2)}
    \end{subfigure}
    \caption{Strategies and communication pattern for different scenarios}
    \label{fig:strategies_and_comm}
\end{figure*}
\subsection{Communication Overhead}
We now evaluate the communication overhead of VBC. To quantify the amount of communication, we run Algorithm~\ref{alg:communication_proto} and count the total number of pairs of agents $g_{t}$ that conduct communication for each timestep $t$, then divided by the total number of pairs of agents in the user group, $R$. For example, for the task 3s$\_$vs$\_$4z, the user controls 3 Stalkers, and therefore the total number of possible pairs of agents $R$ for communication is $3\times 2 = 6$. Within these $6$ pairs of agents, suppose that $2$ pairs involve communication, and therefore $g_{t} = 2$. Let $\beta = \sum_{t} g_{t}/RT$ denote the communication overhead. Table~\ref{table:communication_overhead} shows the $\beta$ of VBC+VDN, VBC+QMIX and SchedNet across all the test episodes at the end of the training phase of each battle. For SchedNet, $\beta$ simply equals the ratio between number of allied agents that are allowed to talk and the total number of allied agents. As shown in Table~\ref{table:communication_overhead}, in contrast to ScheNet, VBC+VDN and VBC+QMIX produce 10$\times$ lower communication overhead for MMM and 2s3z, and $2-6\times$ less traffic for the rest of tasks.

\begin{figure}[bp!]
    \centering
    \begin{subfigure}{0.4\textwidth}
        \centering
        \includegraphics[width=0.9\linewidth, height=0.6\linewidth]{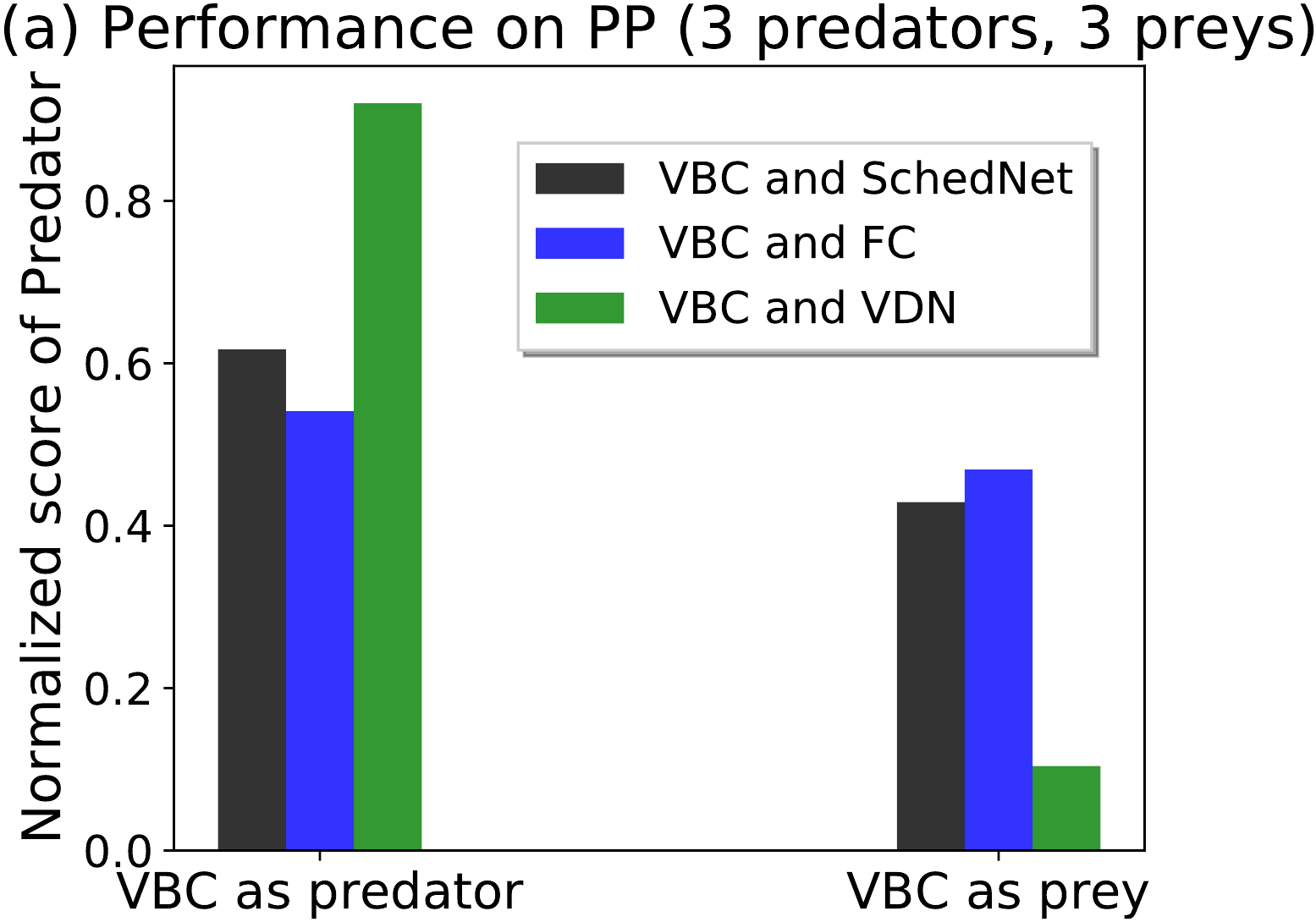}

    \end{subfigure}  
    \begin{subfigure}{0.5\textwidth}
    \captionsetup{skip=0pt,font=scriptsize}
    \subcaption*{(b) Results on Cooperative Navigation (\#agents = 6)}
    \centering
    \begin{tabular}{c|c|c}
    \hline
    \enspace\enspace Methods  &  Avg. dist &  \#collisions \\\hline
    VBC+VDN  & \enspace 2.687 & \enspace\enspace 0.169 \\
    \enspace SchedNet  & \enspace 2.798 & \enspace\enspace 0.176 \\
    \enspace FC        & \enspace 2.990 & \enspace\enspace 0.161 \\
    \enspace\enspace VDN  & \enspace 3.886 & \enspace\enspace 1.872 \\
    \hline
    \end{tabular}
    \end{subfigure}  
    \caption{(a) Results on PP with 3 predators and 3 prey. (b) shows results of CN.}
    \label{fig:maddpg-result2}
\end{figure}

\begin{table}
\centering
\caption{Communication overhead}
\begin{tabular}{|c | c | c | c|}
\hline
$\beta$ & VBC+VDN & VBC+QMIX & SchedNet \\ \hline
\textbf{MMM}  & 5.25\% & 5.36\% & 50\% \\
\textbf{2s3z}  & 4.33\% & 4.68\% & 60\% \\
\textbf{3s5z}  & 27.70\% & 28.13\% & 62.5\% \\
\textbf{3s\_vs\_4z}  & 5.07\% & 5.19\% & 33.3\% \\
\textbf{6h\_vs\_8z}  & 35.93\% & 36.16\% & 66.7\% \\
\textbf{6z\_vs\_24zerg}  & 12.13\% & 13.35\% & 50\% \\\hline
\end{tabular}
\label{table:communication_overhead}
\end{table}

\subsection{Learned Strategy}
\label{sec:strategy}
In this section, we examine the behaviors of the agents in order
to better understand the strategies adopted by the different algorithms. We have made a video demo available at~[2] for better illustration.


For unsymmetrical battles, the number of allied units is less than the enemy units,
and therefore the agents are prone to be besieged and killed by the enemies. This is exactly what happened for the QMIX and VDN agents on 6h$\_$vs$\_$8z, as shown in (Figure~\ref{fig:strategies_and_comm}(c)). 
Figure~\ref{fig:strategies_and_comm}(b) shows the strategy of VBC, all the Hydralisks are placed in a row at the bottom margin of the map. Due to the limited size of the map, the Zealots can not go beyond the margin to surround the Hydralisks. The Hydralisks then focus their firings to kill each Zealot. Figure~\ref{fig:strategies_and_comm}(a) shows the change on $\beta$ for a sample test episode. We observe that most of the communication appears in the beginning of the episode. This is due to the fact that Hydralisks need to talk in order to place in a row. After the arrangement is formed, no communication is needed until the arrangement is broken due to the deaths of some Hydralisks, as indicated by the short spikes near the end of the episode. 
Finally, SchedNet and FC utilize a similar strategy as VBC. Nonetheless, due to the restriction on communication pattern, the row formed by the allied agents are usually not well formed, and can be easily broken by the enemies.

For 3s$\_$vs$\_$4z scenario, the Stalkers have a larger attack range than Zealots. 
All the algorithms adopt a kiting strategy where the Stalkers form a group and attack the Zealots while kiting them. For VBC and FC, at each timestep only the agents that are far from the enemies attack, and rest of the agents (usually the healthier ones) are used as a shield to protect the firing agents (Figure~\ref{fig:strategies_and_comm}(d)). Communication only occurs when the group are broken and need to realign. In contrast, VDN and QMIX do not have this attacking pattern, and all the Stalkers always fire simultaneously, therefore the Stalkers closest to the Zealots are get killed soon. SchedNet and FC also adopt a similar policy as VBC, but the attacking pattern of the Stalkers is less regular, i.e., the Stalkers close to the Zealots also fire occasionally. 

6z$\_$vs$\_$24zerg is the toughest scenario in our experiment. For QMIX and VDN, the 6 Zealots are surrounded and killed by 24 Zerglings shortly after episode starts. In contrast, VBC first separates the agents into two groups with two Zealots and four Zealots respectively (Figure~\ref{fig:strategies_and_comm}(e)). The two Zealots attract most of Zerglings to a place far away from the rest four Zealots, and are killed shortly. Due to the limit sight range of the Zerglings, they can not find the rest four Zealots. On the other side, the four Zealots kill the small part of Zerglings easily and search for the rest Zerglings. The four Zealots take advantage of the short sight of the Zerglings. 
Each time the four Zealots adjust their positions in a way such that they can only be seen by a small amount of the Zerglings, the baited Zerglings are then be killed easily (Figure~\ref{fig:strategies_and_comm}(f)). For VBC, the communication only occurs in the beginning of the episode when the Zealots are separated into two groups, and near the end of the episode when four Zealots adjust their positions. Both FC and SchedNet learn the strategy of splitting the Zealots into two groups, but they fail to fine-tune their positions to kill the remaining Zerglings. 

For symmetrical battles, the tasks are less challenging, and we see less disparities on performances of the algorithms. For 2s3z and 3s5z, the VDN agents attack the enemies blindly without any cooperation. The QMIX agents learn to focus firing and protect the Stalkers. The agents of VBC, FC and SchedNet adopt a more aggressive policy, where the allied Zealots try to surround and kill the enemy Zealots first, and then attack the enemy Stalkers by collaborating with the allied Stalkers. This is extremely effective because Zealots counter Stalkers, so it is important to kill the enemy Zealots before they damage allied Stalkers. For VBC, the communication occurs mostly when the allied Zealots try to surround the enemy Zealots. For MMM, almost all the methods learn the optimal policy, namely killing the Medivac first, then attack the rest of the enemy units cooperatively.
\subsection{Evaluation on Cooperative Navigation and Predator-prey}

Furthermore, we evaluate the algorithms for two more scenarios: (1) Cooperative Navigation (CN) which is a cooperative scenario, and (2) Predator-prey (PP) which is a competitive scenario. The game settings are the same as what are used in~[10] and~[8], respectively. We train each method until convergence and test the result models for 2000 episodes. For PP, we make the agents of VBC to compete against the agents of other methods, and report the normalized score of Predator (Figure~\ref{fig:maddpg-result2}(a)). For CN we report the average distance between agents and their destinations, and average number of collisions (Figure~\ref{fig:maddpg-result2}(b)). We notice that methods which allows communication (\ie, SchedNet, FC, VBC) outperform the others for both tasks, and VBC achieves the best performance. Moreover, VBC incurs a communication overhead of $10.07\%$ and $8.80\%$ for PP and CN respectively. In CN, most of the communication of VBC occurs when the agents are close to each other to prevent collisions. In PP, the communication of VBC occurs mainly to rearrange agent positions for better coordination.

\section{Conclusion}
In this work, we propose VBC, a simple and effective approach to achieve efficient communication among the agents in MARL. By constraining the variance of the exchanged messages during the training phase, VBC improves communication efficiency while enabling better cooperation among the agents.
The test results of StarCraft Multi-Agent Challenge indicate that VBC outperforms the other state-of-the-art methods significantly in terms of both winning rate and communication overhead.


\newpage
\section{Appendix}

\subsection{Proof for theorem 1}
\begin{theorem}
    Assume $0\leq \eta_{k} \leq 1$, $\sum_{k}\eta_{k}=\infty$, $\sum_{k}\eta_{k}^{2}<\infty$. Also assume the number of possible actions and states are finite. By performing equation 2 iteratively, we have $||Q_{tot}^{k}(\textbf{o}_{t},\textbf{a}_{t})-Q_{tot}^{*}(\textbf{o}_{t},\textbf{a}_{t})||\leq \lambda NG$ $\forall \textbf{o}_{t},\textbf{a}_{t}$, as $k \rightarrow \infty$, where $G$ satisfies $||\frac{\partial Var(f_{enc}(c^{t}_{i}))}{\partial Q_{tot}^{k}(\textbf{o}_{t},\textbf{a}_{t})}||\leq G, \forall i,k,t,\textbf{o}_{t},\textbf{a}_{t}$.
\end{theorem}
\begin{proof}
The proof is based on~[21] and ~[22]. 
If we subtract $Q_{tot}^{*}(\textbf{o}_{t},\textbf{a}_{t})$ from equation 2 in the paper, and rearrange the equation, we get:
\begin{equation}
        \delta^{k+1}(\textbf{o}_{t},\textbf{a}_{t}) = (1-\eta_{k})\delta^{k}(\textbf{o}_{t},\textbf{a}_{t}) + \eta_{k}\bigg[r_{t} + \gamma max_{\textbf{a}}Q_{tot}^{k}(\textbf{o}_{t+1},\textbf{a})
         - Q_{tot}^{*}(\textbf{o}_{t},\textbf{a}_{t}) - \lambda\sum_{i=1}^{N}\frac{\partial Var(f_{enc}(c^{t}_{i}))}{\partial Q_{tot}^{k}(\textbf{o}_{t},\textbf{a}_{t})}\bigg] 
\end{equation}
where $\delta^{k}(\textbf{o}_{t},\textbf{a}_{t}) = Q_{tot}^{k}(\textbf{o}_{t},\textbf{a}_{t})-Q_{tot}^{*}(\textbf{o}_{t},\textbf{a}_{t})$. Let $r_{t} + \gamma max_{\textbf{a}}Q_{tot}^{k}(\textbf{o}_{t+1},\textbf{a})- Q_{tot}^{*}(\textbf{o}_{t},\textbf{a}_{t}) = F_{k}(\textbf{o}_{t},\textbf{a}_{t})$, and $\sum_{i=1}^{N}\frac{\partial Var(f_{enc}(c^{t}_{i}))}{\partial Q_{tot}^{k}(\textbf{o}_{t},\textbf{a}_{t})} = U_{k}(\textbf{o}_{t},\textbf{a}_{t})$ we have:
\begin{equation}
        \delta^{k+1}(\textbf{o}_{t},\textbf{a}_{t}) = (1-\eta_{k})\delta^{t}(\textbf{o}_{t},\textbf{a}_{t}) + \eta_{k}\bigg[F_{k}(\textbf{o}_{t},\textbf{a}_{t}) - \lambda U_{k}(\textbf{o}_{t},\textbf{a}_{t})\bigg] 
\end{equation}
Decompose $\delta^{k+1}(\textbf{o}_{t},\textbf{a}_{t})$ into two random processes, $\delta^{k+1}_{1}(\textbf{o}_{t},\textbf{a}_{t})$ and $\delta^{k+1}_{2}(\textbf{o}_{t},\textbf{a}_{t})$, where $\delta^{k+1}(\textbf{o}_{t},\textbf{a}_{t}) = \delta^{k+1}_{1}(\textbf{o}_{t},\textbf{a}_{t}) + \delta^{k+1}_{2}(\textbf{o}_{t},\textbf{a}_{t})$, we have the following two random iterative processes:
\begin{align}
    \delta^{k+1}_{1}(\textbf{o}_{t},\textbf{a}_{t}) &= (1-\eta_{k})\delta^{k}_{1}(\textbf{o}_{t},\textbf{a}_{t}) + \eta_{k}F_{k}(\textbf{o}_{t},\textbf{a}_{t}) \\
    \delta^{k+1}_{2}(\textbf{o}_{t},\textbf{a}_{t}) &= (1-\eta_{k})\delta^{k}_{2}(\textbf{o}_{t},\textbf{a}_{t}) - \eta_{k}\lambda U_{k}(\textbf{o}_{t},\textbf{a}_{t}) 
\end{align}
From Theorem 1 of~[22], we know that $\delta^{k+1}_{1}(\textbf{o}_{t},\textbf{a}_{t})$ converges to zero w.p. 1. From equation $6$, we notice that:
\begin{align}
     ||\delta^{k+1}_{2}(\textbf{o}_{t},\textbf{a}_{t})|| &\leq ||(1-\eta_{k})\delta^{k}_{2}(\textbf{o}_{t},\textbf{a}_{t})|| + \eta_{k}\lambda ||U_{t}(\textbf{o}_{t},\textbf{a}_{t})|| \\
     &\leq(1-\eta_{k})||\delta^{k}_{2}(\textbf{o}_{t},\textbf{a}_{t})|| + \eta_{k}\lambda NG
\end{align}
Therefore we have $||\delta^{k+1}_{2}(\textbf{o}_{t},\textbf{a}_{t})||-\lambda NG \leq(1-\eta_{k})(||\delta^{k}_{2}(\textbf{o}_{t},\textbf{a}_{t})|| -\lambda NG)$.
This is linear in $\delta^{k+1}_{2}(\textbf{o}_{t},\textbf{a}_{t})$ and $||\delta^{k+1}_{2}(\textbf{o}_{t},\textbf{a}_{t})||-\lambda NG$ will converge to a nonpositive number as k approaches infinity. However, because $||\delta^{k+1}_{2}(\textbf{o}_{t},\textbf{a}_{t})||$ is always greater or equal to zero, hence $||\delta^{k+1}_{2}(\textbf{o}_{t},\textbf{a}_{t})||$ must converge to a number between 0 and $\lambda NG$.
Therefore we get:
\begin{align}
     & ||Q_{tot}^{k}(\textbf{o}_{t},\textbf{a}_{t})-Q_{tot}^{*}(\textbf{o}_{t},\textbf{a}_{t})|| =||\delta^{k}(\textbf{o}_{t},\textbf{a}_{t})|| \\
     &=||\delta^{k}_{1}(\textbf{o}_{t},\textbf{a}_{t}) + \delta^{k}_{2}(\textbf{o}_{t},\textbf{a}_{t})|| \\
     &\leq ||\delta^{k}_{1}(\textbf{o}_{t},\textbf{a}_{t})|| + ||\delta^{k}_{2}(\textbf{o}_{t},\textbf{a}_{t})|| \\
     &\leq \lambda NG 
\end{align}
as k approaches infinity.
\end{proof}

\section{Experiment settings and hyperparameters}
In this section, we describe in detail the experiment settings.
\subsection{StarCraft micromanagement challenges}
StarCraft II~[1] is a real-time strategy (RTS) game that has recently been utilized as a challenging benchmark by the reinforcement learning community~[14,5,13]. In this work, we concentrate on the~\emph{decentralized micromanagement problem} in StarCraft II. Specifically, the user controls an army that consists of several army units (agents). We consider the battle scenario where two armies, one controlled by the user, and the other controlled by the build-in StarCraft II AI, are placed on the same map and try to defeat each other. The agent type can be different between the two armies, and the agent type can also be different within the same army. The goal of the user is to control the allied units wisely to kill all the enemy units, while minimizing the damage on the health of each individual agent. The difficulty of the computer AI is set to \emph{Medium}. We consider six different battle settings, which is shown in Table~\ref{table:game_setting}. Among these six settings, three are \emph{symmetrical battles}, where the user group and the enemy group are identical in terms of type and quantity of agents. The other three are \emph{unsymmetrical battles}, where the agent type of user groups and enemy group are different, and user group contains less number of allied units than the enemy group. The unsymmetrical battles is consider to be harder than the symmetrical battles because of the difference in the army size.

During execution, each agent are allowed to perform the following actions: move[direction], attack[enemy$\_$id], stop and no-op. There are four directions for the 'move' operation: east, west, south, or north. Each agent can only attack the enemy within its \emph{shooting range}. The Medivacs can only heal the partner by performing the heal[partner$\_$id] instead of attacking the enemies. The number of possible actions for an agent ranges from 11 (2s3z) to 30 (6z$\_$vs$\_$24zerg). Each agent has a limited \emph{sight range}, and can only receive information from the partners or enemies within the sight range. Furthermore, the shooting range is smaller than the sight range so the agent can not attack the opponents without observing them, and the attack operation is not available when the opponents are outside the shooting range. Finally, agents can only observe the live agents within the sight range, and can not discern the agents that are dead or outside the range. At each timestep, the joint reward received by the allied units equals to the total damage on the health levels of the enemy units. Additionally, the agents are rewarded 100 extra points after killing each enemy unit, and 200 extra points for killing all the enemies. The user group wins the battle only when the allied units kill all the enemies within the time limit. The user group loses the battle if either all the allied units are killed, or the time limit reaches. The time limit for different battles are: 120 timesteps for 2s3z and MMM, 150 for 3s5z and 3s$\_$vs$\_$4z, and 200 for 6h$\_$vs$\_$8z and 6z$\_$vs$\_$24zerg. 

The input observation of each agent consisting of a vector which involves the following information of each allied unit and enemy unit in its sight range: relative x,y coordinates, relative distance and agent type. For the mixing network of QMIX, the global state vector $\textbf{s}_{t}$ contains the following elements: 
\begin{enumerate}
    \item Shield levels, health levels and cooldown levels of all the units at $t$. 
    \item The actions taken by all the units at $t-1$.
    \item The x,y coordinates of all the units relative to the center of the map at $t$.
\end{enumerate}
For all the six battles, each allied or enemy agent has a sight range of 9 and shooting range of 6 for all types of agents. For additional information, please refer to~[15].
\begin{table*}
\small
\centering
\caption{Agent types of the six battles}
\begin{tabular}{c|c|c|c}

\textbf{Symm.} & \textbf{2s3z} & \textbf{MMM} & \textbf{3s5z}  \\ \hline
User  & 2 Stalkers \& 3 Zealots & 1 Medivac, 2 Marauders \& 7 Marines & 3 Stalkers \& 5 Zealots \\\hline
Enemy  & 2 Stalkers \& 3 Zealots & 1 Medivac, 2 Marauders \& 7 Marines & 3 Stalkers \& 5 Zealots \\\hline
\textbf{Unsymm.} & \textbf{3s\_vs\_4z} & \textbf{6h\_vs\_8z} & \textbf{6z\_vs\_24zerg} \\ \hline
User  & 3 Stalkers  &  6 Hydralisks & 6 Zealots \\\hline
Enemy  & 4 Zealots &   8 Zealots &  24 Zerglings     
\end{tabular}
\label{table:game_setting}
\end{table*}
\subsection{Hyperparameter}
For network of agent $i$, at timestep $t$, raw observation $o_{i}^{t}$ is first passed through a single-layer MLP, which outputs a intermediate result with size of 64. The GRU then takes this intermediate result, as well as the hidden states $h^{t-1}_{i}$ from the previous timestep and generates $h^{t}_{i}$ and $c_{i}^{t}$. Both $h^{t}_{i}$ and $c_{i}^{t}$ has a size of 64.  The $c_{i}^{t}$ is then passed through a FC layer, which generates the local action value function $Q_{i}(o_{i}^{t},h_{i}^{t},.)$. The message encoder contains two FC layers with 196 and 64 units respectively. The combiner performs elementwise summation on the outputs of local action generator and the message encoders.

During the training, we set the $\gamma = 0.99$ and decrease $\epsilon$ linearly from 1.0 to 0.05 over the first $200000$ timesteps and keep it to $0.05$ for the rest of the learning process. The replay buffers stores the most recent 5000 episode. We perform a test run for every 200 training episodes to update the replay buffer. The training batch size is set to 32 and the test batch size is set to 8. We adopt a RMSprop optimizer with a learning rate $\eta=5\times 10^{-4}$ and $\alpha = 0.99$.

\section*{Reference}
{[1]}\enspace Starcraft official game site: \url{https://starcraft2.com/}. 

{[2]}\enspace VBC Video demo: \url{https://bit.ly/2VFkvCZ}. 

{[3]}\enspace J.~Foerster, I.~A. Assael, N.~de~Freitas, and S.~Whiteson.
\newblock Learning to communicate with deep multi-agent reinforcement learning.
\newblock In {\em Advances in Neural Information Processing Systems}, pages
  2137--2145, 2016.
  
{[4]}\enspace J.~N. Foerster, C.~A.~S. de~Witt, G.~Farquhar, P.~H. Torr, W.~Boehmer, and
  S.~Whiteson.
\newblock Multi-agent common knowledge reinforcement learning.
\newblock {\em arXiv preprint arXiv:1810.11702}, 2018.

{[5]}\enspace J.~N. Foerster, G.~Farquhar, T.~Afouras, N.~Nardelli, and S.~Whiteson.
\newblock Counterfactual multi-agent policy gradients.
\newblock In {\em Thirty-Second AAAI Conference on Artificial Intelligence},
  2018.
  
{[6]}\enspace M.~Hausknecht and P.~Stone.
\newblock Deep recurrent q-learning for partially observable mdps.
\newblock In {\em 2015 AAAI Fall Symposium Series}, 2015.

{[7]}\enspace J.~Jiang and Z.~Lu.
\newblock Learning attentional communication for multi-agent cooperation.
\newblock In {\em Advances in Neural Information Processing Systems}, pages
  7254--7264, 2018.
  
{[8]}\enspace D.~Kim, S.~Moon, D.~Hostallero, W.~J. Kang, T.~Lee, K.~Son, and Y.~Yi.
\newblock Learning to schedule communication in multi-agent reinforcement
  learning.
\newblock {\em arXiv preprint arXiv:1902.01554}, 2019.

{[9]}\enspace J.~Kober, J.~A. Bagnell, and J.~Peters.
\newblock "reinforcement learning in robotics: A survey.".
\newblock {\em The International Journal of Robotics Research}, 2013.

{[10]}\enspace R.~Lowe, Y.~Wu, A.~Tamar, J.~Harb, O.~P. Abbeel, and I.~Mordatch.
\newblock Multi-agent actor-critic for mixed cooperative-competitive
  environments.
\newblock In {\em Advances in Neural Information Processing Systems}, pages
  6379--6390, 2017.
  
{[11]}\enspace L.~Marc, V.~Zambaldi, A.~Gruslys, A.~Lazaridou, K.~Tuyls, J.~Pérolat,
  D.~Silver, and T.~Graepel.
\newblock "a unified game-theoretic approach to multiagent reinforcement
  learning.".
\newblock In {\em Advances in Neural Information Processing Systems}, 2017.

{[12]}\enspace V.~Mnih, K.~Kavukcuoglu, D.~Silver, A.~Graves, I.~Antonoglou, D.~Wierstra, and
  M.~Riedmiller.
\newblock "playing atari with deep reinforcement learning.".
\newblock {\em arXiv preprint arXiv:1312.5602}, 2013.

{[13]}\enspace P.~Peng, Y.~Wen, Y.~Yang, Q.~Yuan, Z.~Tang, H.~Long, and J.~Wang.
\newblock Multiagent bidirectionally-coordinated nets: Emergence of human-level
  coordination in learning to play starcraft combat games.
\newblock {\em arXiv preprint arXiv:1703.10069}, 2017.

{[14]}\enspace T.~Rashid, M.~Samvelyan, C.~S. de~Witt, G.~Farquhar, J.~Foerster, and
  S.~Whiteson.
\newblock Qmix: Monotonic value function factorisation for deep multi-agent
  reinforcement learning.
\newblock {\em arXiv preprint arXiv:1803.11485}, 2018.

{[15]}\enspace M.~Samvelyan, T.~Rashid, C.~S. de~Witt, G.~Farquhar, N.~Nardelli, T.~G.~J.
  Rudner, C.-M. Hung, P.~H.~S. Torr, J.~Foerster, and S.~Whiteson.
\newblock {The} {StarCraft} {Multi}-{Agent} {Challenge}.
\newblock {\em CoRR}, abs/1902.04043, 2019.

{[16]}\enspace S.-S. Shai, S.~Shammah, and A.~Shashua.
\newblock "safe, multi-agent, reinforcement learning for autonomous driving.".
\newblock {\em arXiv preprint arXiv:1610.03295}, 2016.

{[17]}\enspace S.~Sukhbaatar, R.~Fergus, et~al.
\newblock Learning multiagent communication with backpropagation.
\newblock In {\em Advances in Neural Information Processing Systems}, pages
  2244--2252, 2016.
  
{[18]}\enspace P.~Sunehag, G.~Lever, A.~Gruslys, W.~M. Czarnecki, V.~Zambaldi, M.~Jaderberg,
  M.~Lanctot, N.~Sonnerat, J.~Z. Leibo, K.~Tuyls, et~al.
\newblock Value-decomposition networks for cooperative multi-agent learning.
\newblock {\em arXiv preprint arXiv:1706.05296}, 2017.

{[19]}\enspace A.~Tampuu, T.~Matiisen, D.~Kodelja, I.~Kuzovkin, K.~Korjus, J.~Aru, J.~Aru, and
  R.~Vicente.
\newblock Multiagent cooperation and competition with deep reinforcement
  learning.
\newblock {\em PloS one}, 12(4):e0172395, 2017.

{[20]}\enspace M.~Tan.
\newblock "multi-agent reinforcement learning: Independent vs. cooperative
  agents.".
\newblock In {\em Proceedings of the tenth international conference on machine
  learning.} IEEE, 1993.

{[21]}\enspace T.~Jaakkola, M.~I. Jordan, and S.~P. Singh.
\newblock Convergence of stochastic iterative dynamic programming algorithms.
\newblock In {\em Advances in neural information processing systems}, pages
  703--710, 1994.

{[22]}\enspace F.~S. Melo.
\newblock Convergence of q-learning: A simple proof.
\newblock {\em Institute Of Systems and Robotics, Tech. Rep}, pages 1--4, 2001.

\end{document}